\documentclass{article}

\usepackage{microtype}
\usepackage{graphicx}
\usepackage{booktabs} 

\usepackage{amsfonts}
\usepackage{todonotes}
\usepackage{physics}
\usepackage{amsthm}

\usepackage{subcaption}
\newcommand{\E}{\mathbb{E}}

\newcommand{\Ind}{\textbf{1}}

\newcommand{\V}{\mathbb{V}}
\newcommand{\C}{\mathcal{C}}
\newcommand{\Ck}{\mathcal{C}_k}
\newcommand{\notCk}{\bar{\mathcal{C}}_k}

\newtheorem{theorem}{Theorem}

\newtheorem{proposition}[theorem]{Proposition}

\allowdisplaybreaks

\DeclareMathOperator*{\argmin}{arg\,min}

\usepackage{hyperref}

\usepackage[accepted]{icml2019}

\begin{document}

\twocolumn[
\icmltitle{Rao-Blackwellized Stochastic Gradients for Discrete Distributions}




\begin{icmlauthorlist}
\icmlauthor{Runjing Liu}{a}
\icmlauthor{Jeffrey Regier}{b}
\icmlauthor{Nilesh Tripuraneni}{b}
\icmlauthor{Michael I. Jordan}{a,b}
\icmlauthor{Jon McAuliffe}{a,c}
\end{icmlauthorlist}

\icmlaffiliation{a}{Department of Statistics, University of California, Berkeley}
\icmlaffiliation{b}{Department of Electrical Engineering and Computer Sciences, University of California, Berkeley}
\icmlaffiliation{c}{The Voleon Group}

\icmlcorrespondingauthor{Runjing Liu}{runjing\_liu@berkeley.edu}

\icmlkeywords{Machine Learning, ICML}

\vskip 0.3in
]



\printAffiliationsAndNotice{}  

\begin{abstract}
We wish to compute the gradient of an expectation over a finite or countably infinite sample space having $K \leq \infty$ categories. When $K$ is indeed infinite, or finite but very large, the relevant summation is intractable. Accordingly, various stochastic gradient estimators have been proposed. In this paper, we describe a technique that can be applied to reduce the variance of any such estimator, without changing its bias---in particular, unbiasedness is retained. We show that our technique is an instance of Rao-Blackwellization, and we demonstrate the improvement it yields on a semi-supervised classification problem and a pixel attention task.

\end{abstract}

\section{Introduction}\label{sec:introduction}
Let $z$ be a discrete random variable over $K \leq \infty$ categories, with distribution $q_\eta(z)$ parameterized by a real vector $\eta$ and differentiable in $\eta$. We aim to minimize 
\begin{align}
\mathcal L(\eta) := \mathbb E_{z \sim q_\eta(z)} \left[ f_\eta(z) \right] = \sum_{k = 1}^K q_\eta(k) f_\eta(k),
\label{eq:objective_fun}
\end{align}
where the real-valued integrand $f_\eta$ also depends differentiably on $\eta$. If $K$ is finite and small enough, we can compute the exact gradient as 

\begin{align}
 \begin{split}
\nabla_\eta &\E_{q_\eta(z)}[f_\eta(z)]\\
&=  \sum_{k = 1}^K \Big\{\left[\nabla_\eta q_\eta(k) \right] f_\eta(k) 
 + q_\eta(k)\nabla_\eta  f_\eta(k) \Big\}.
\label{eq:analytic_gradient}
\end{split}
\end{align}

On the other hand, $K$ may be infinite, or large relative to the cost of evaluating $q_\eta \cdot f_\eta$. In either of these cases, which are the focus of this paper, the exact gradient is computationally intractable. Thus, in order to optimize $\mathcal L(\eta)$, we seek low-variance stochastic approximations of the gradient. 

The ``reparametrization trick''~\cite{SpallOptimization2003, kingma2014auto} provides efficient stochastic gradients when $q_\eta$ is a continuous distribution, but it does not apply when $z$ is discrete. Two well-known possibilities in the discrete case are continuous relaxation~\cite{maddison2017concrete,jang2017categorical} and REINFORCE~\cite{williams1992simple} (also known as the score function estimator). The former replaces the discrete random variable with a continuous relaxation so that the reparametrization trick can be applied. However, it results in biased gradient estimates. The latter is impractical for most purposes due to its high variance. 

Control variate methodology provides a general framework for variance reduction.  Specific examples include RELAX~\cite{grathwohl2018backpropagation}, REBAR~\cite{tucker2017rebar}, 
NVIL~\cite{mnih2014neural},  and MuProp~\cite{GuMuProp2015}.  These methods provide a mechanism for reducing the variance of REINFORCE, but unfortunately they do not reduce the variance enough for many applications.


In the current paper, we show how to achieve further variance reduction via a meta-procedure that can be applied to any discrete-distribution stochastic-gradient procedure (e.g., REINFORCE or REINFORCE with control variate).  Our framework reduces variance without changing the bias. In particular, an unbiased stochastic gradient remains unbiased after application of our approach. Further, our approach is ``anytime'' in the sense that it can reduce stochastic-gradient variances given any computational budget---the larger the budget, the greater the variance reduction. Hence it is well suited to our chosen setting, where $K$ is infinite or very large, and/or $q_\eta \cdot f_\eta$ is expensive to evaluate.

Our method is particularly apt in the setting where the probability mass $q_\eta(z)$ is concentrated on only a few categories. For example, in extreme classification, only a few labels out of many are plausible. In reinforcement learning, only a few actions in the possible action space are advantageous. Neither control-variate methods nor continuous-relaxation techniques take advantage of this ``sparsity,'' and we show that the variance reduction of our method in this setting can be dramatic. 

We show that our variance-reduction meta-procedure is an instance of a general statistical method called Rao-Blackwellization~\cite{casella1996rao}. Rao-Blackwellization has been used in previous work to reduce the variance of stochastic gradients~\cite{bbVI, TitsiasLocalExpGrads}, but in a setting orthogonal to ours, one with multivariate discrete random variables. Our focus here is on a univariate discrete random variable with many categories. Our method can be applied in conjunction with the former work to extend to the case of multivariate discrete random variables, each with a large number of categories. This extension is not discussed in the present work, and we leave it as an avenue of future exploration. 


The paper is organized as follows. We present our variance-reduction procedure in Section~\ref{sec:methods} and make the connection to Rao-Blackwellization in Section~\ref{sec:theory}, demonstrating that our technique necessarily reduces stochastic-gradient variances. In Section 4 we discuss related work. In Section 5, we exhibit the benefits of our procedure on synthetic data, a semi-supervised classification problem, and a pixel attention task. We conclude in Section 6.

\section{Method}\label{sec:methods}
We consider the situation where the number of categories $K$ is infinite, or very large in the sense that computing the exact gradient in Equation \eqref{eq:analytic_gradient} is intractable. One possible stochastic estimator for the gradient is the REINFORCE estimator, 
\begin{align}
    f_\eta(z) \nabla_\eta \log q_\eta(z) + \nabla_\eta f_\eta(z) \quad 
    z\sim q_\eta(z),
    \label{eq:reinforce_estimate}
\end{align}
which one can check is unbiased for the true gradient in Equation \eqref{eq:analytic_gradient}. 

In practice, the REINFORCE estimator often has variance too large to be useful. Control variates have been proposed to decrease the variance of the REINFORCE estimator. The key observation is that the score function $\nabla_\eta \log q_\eta(z)$ has zero expectation under $q_\eta(z)$, so 
\begin{align}
    [f_\eta(z) - C]\nabla_\eta \log q_\eta(z)  + \nabla_\eta f_\eta(z) \quad z \sim q_\eta(z)
    \label{eq:reinforce_w_control_estimate}
\end{align}
is still unbiased for the true gradient. Several proposals have been put forth for choosing $C$ to reduce the variance~\cite{mnih2014neural,GuMuProp2015,tucker2017rebar}. 

In this paper, we present a meta-procedure that can be applied to any stochastic estimator for the gradient of a discrete expectation obtained by sampling from $q_\eta(z)$. Let $g(z)$ 
be any such estimator which is unbiased\footnote{Our technique applies to biased estimators as well. For concreteness, we focus on the unbiased case.}, i.e., satisfies $\E_{q_{\eta}(z)}[g(z)] = \nabla_{\eta} \E_{q_{\eta}(z)} [f_\eta(z)]$. An example is the REINFORCE estimator. We decompose the expectation $\E_{q_{\eta}(z)}[g(z)]$ into two components: one containing the high-probability atoms of $q_\eta$, and one containing the remaining atoms. We compute the exact contribution of the high-probability component to the expectation, and we use a stochastic estimator for the other component. The idea comes from observing that in many applications, $q_\eta(z)$ only puts significant mass on a few categories. If $g(z)$ is reasonably well behaved over $z$, then $q_\eta(z)g(z)$ is large when $q_\eta(z)$ attains its largest values and smaller elsewhere. By computing the high-probability component of the expectation exactly, we obtain a value already close to correct. A stochastic estimator is then added to correct, on average, for what error remains. 




Formally, let $\mathcal{C}_k$ be the set of $z$ such that $q_\eta(z)$ assumes one of its $k$ largest values. Ties may be broken arbitrarily. Let $\bar{\C}_k$ denote the complement of $\C_k$. Then
\begin{align}
 \nabla_\eta &\E_{q_\eta(z)}[f_\eta(z)] = \E_{q_\eta(z)}[g(z)] \label{eq:unbiasedness}\\
 &= \E_{q_\eta(z)}[g(z) \Ind\{z \in \C_k\} + g(z) \Ind\{z \in \bar{\C}_k\}] \\
&= \sum_{z \in \C_k}q_\eta(z) g(z) + \E_{q_\eta(z)}[g(z)\Ind\{z \in \bar{\C}_k\}].
\label{eq:partial_marg_decomp}
\end{align}
It remains to approximate the expectation in the second term. We use an importance-sampling approximation based on a single draw from an importance distribution.
We choose a simple importance distribution: the distribution of $q_\eta$ conditional on the event $\notCk$. We denote this importance distribution by $q_\eta|_{ \bar{\C}_k}$. By construction, the importance weighting function is identically equal to $q_\eta(\notCk)$, regardless of which $z \sim q_\eta|_{ \bar{\C}_k}$ we draw. (Note that the indicator inside the second term of~\eqref{eq:partial_marg_decomp} always equals one, because we are only sampling from $z\in \notCk$.)

Our estimator assumes that, given $k$, the set $\C_k$ can be identified at little cost. This certainly holds in the case of inference: using variational Bayes, $q(z)$ is a variational approximate posterior chosen from a set we designate. 

In summary, we estimate the gradient as
\begin{align}
    & \hat g(v) = \sum_{z \in \Ck} q_{\eta}(z) g(z) + q_{\eta}(\notCk) g(v) \label{eq:rb_estimate} \\
    & \quad v \sim q_\eta|_{\notCk} \notag,
\end{align}
which also satisfies $\E_{v}[\hat g (v)] = \nabla_{\eta} \E_{q_{\eta}(z)} [f_\eta(z)]$.

We see that the first term of this estimator is deterministic and the second term is random, but scaled by $q_\eta(\notCk)$, which is small when $q_\eta$ is concentrated on a small number of atoms. Therefore, we intuitively expect this estimator to have smaller variance than the original estimator, $g(z)$. 

In the next section, we confirm this intuition by interpreting the construction of the estimator $\hat g(v)$ as Rao-Blackwellization (which always reduces variance). Hence, we call $\hat g(v)$ the {\itshape Rao-Blackwellized gradient estimator}. 


\section{Theory}\label{sec:theory}
We begin by describing how a suitable representation of the original discrete variable $z\sim q_\eta(z)$ allows us to interpret our estimator as an instance of Rao-Blackwellization. Let $q_\eta|_{\Ck}$ denote the distribution of $q_\eta$ conditional on the event $C_k$. Consider the three independent random variables
\begin{align}
    u & \sim q_\eta|_{\Ck}, \\
    v & \sim q_\eta|_{\notCk}, \label{eq:cond_not} \\
\text{and\ \ } b & \sim \textrm{Bernoulli}\left( q_\eta(\notCk)\right). \label{eq:ind_rv}
\end{align}
The triplet $(u, v, b)$ provides a distributionally equivalent representation of $z$:
\begin{align}
    T(u, v, b) \stackrel{d}{=} z,
\end{align}
where
\begin{align}
    T(u, v, b) := u^{1-b} v^b.
\end{align}

The estimator in Equation~\eqref{eq:rb_estimate} can then be written as
\begin{align}
    \hat g(v) = \E \left[ g(T(u, v, b)) | v \right], \label{eq:rbequiv}
\end{align}
where $g(z)$ is the original unbiased gradient estimator. To see this, break the right-hand side of~\eqref{eq:rbequiv} into two terms according to the value of $b$, then simplify. Equation~\eqref{eq:rbequiv} demonstrates directly that our estimator is an instance of Rao-Blackwellization. As such, it has the same expectation as the original estimator $g(z)$, a fact about Rao-Blackwellization that follows immediately from iterated expectation. In particular, if $g(z)$ is unbiased as we have assumed, so too is our estimator.

An application of the conditional variance decomposition
gives
\begin{align}
     \V\left[ g(z) \right] = &\V\left[ \hat g(v) \right]
     + \E\left\{ \V \left[ g(T(u, v, b)) | v \right] \right\}, \label{eq:rb}
\end{align}
showing that $\hat g$ has lower variance than $g$, by at least as much as the last term in Equation~\eqref{eq:rb}. This too is a standard result about Rao-Blackwellization.

Proposition~\ref{prop:var_red} further quantifies this variance reduction, showing the variance of $\hat g(v)$ must be less then the variance of $g(v)$ by the multiplicative factor $q_{\eta}(\notCk)$.
\begin{proposition}
    Let $g(z)$ be an unbiased gradient estimator as in Equation~\eqref{eq:unbiasedness} and $\hat g(v)$ denote the Rao-Blackwellized estimator defined in Equation \eqref{eq:rb_estimate}. Then
    \begin{align}
        \V[\hat g(v)] \leq q_{\eta}(\notCk) \V[g(z)].
    \end{align}
    \label{prop:var_red}
\end{proposition}
\begin{proof}
\vspace{-0.2in}
We apply the conditional variance decomposition.
Let $\epsilon = q_{\eta}(\notCk)$ and recall the Bernoulli random variable $b$ defined in Equation~\eqref{eq:ind_rv}. First,
\begin{align}
    \V[g(z)]
    &= \E[\V[g(z) | b]] + \V[\E[g(z) | b]] \\
    &\geq \E[\V[g(z) | b]] \\
    &= \epsilon \V[ g(z) | z \in \notCk] + (1-\epsilon) \V[ g(z) | z \in \Ck] \notag \\
    &\geq \epsilon \V[g(z) | z \in \notCk]. \notag
\end{align}
But $\V[\hat g(v)] = \epsilon^2 \V[g(z) | z \in \notCk]$, which in combination with the above yields the result.
\end{proof}

The multiplicative factor of variance reduction guaranteed by  Rao-Blackwellization can be significant if the probability mass of $q_{\eta}(z)$ is concentrated on just a few categories. But while Rao-Blackwellization reduces the variance of $g(z)$, this comes at the cost of evaluating $g(z)$ a total $k+1$ times (cf.\ Equation~\eqref{eq:rb_estimate}). An initial stochastic gradient $g(z)$ such as REINFORCE will only require a single evaluation of $g$.

There is an alternative approach to reducing the variance of an initial estimator $g(z)$ via multiple evaluations of $g(z)$: minibatching, i.e., simple Monte-Carlo averaging over independent draws of $z$. Thus, the question arises: given a budget of $N < K$ evaluations of $g(z)$, is it better to Rao-Blackwellize or minibatch? Computationally, our method is parallelizable in the same way that minibatching is parallelizable. The next proposition shows constructively that there is a choice of $k \leq N$ for which Rao-Blackwellization reduces variance at least as much as minibatching.
\begin{proposition}
Suppose we have a budget of $N$ evaluations of $g$. Consider the estimators 
\begin{align}
& \hat g_{N, k}(v) := \sum_{u \in \C_k} q_\eta(u)g(u) + \frac{q_\eta(\bar{\C}_k)}{N-k}\sum_{j=1}^{N-k} g(v_j), \\
&\quad v_1, ..., v_{N-k}  \overset{iid}{\sim} q_{\eta} |_{\notCk} \notag
\end{align}
and
\begin{align}
    & g_N(z) := \frac{1}{N} \sum_{j=1}^N g(z_j), \quad z_1, ..., z_N \overset{iid}{\sim} q_{\eta}.
\end{align}
If we choose
\begin{align}
\hat k = \argmin_{k \in \{0, \hdots, N \} } \frac{q_\eta(\bar{\C}_k)}{N-k}
\end{align}
then $\V[\hat g_{N, \hat{k}}(v)] \leq \V[g_N(z)]$. 
\label{prop:var_compare}
\end{proposition}
\begin{proof}
\vspace{-0.1in}
    Let $V_1 = \V[g_{1}(z)]$. Then $\V[g_N(z)] = (1/N) V_1$, while Proposition \ref{prop:var_red} shows that $\V[\hat g_{N, k}(v)] \leq \frac{q_\eta(\bar{\C}_k)}{N-k} V_1$. Since $\frac{q_{\eta}(\bar{\C}_k)}{N-k} = \frac{1}{N}$ when $k=0$, the result follows.
\end{proof}
Together, Propositions \ref{prop:var_red} and \ref{prop:var_compare} imply the following:
\begin{itemize}
    \item Rao-Blackwellization leads to a significant variance reduction if the mass of $q_{\eta}(z)$ is concentrated.
    \item Even when restricting minibatched versions of the initial and Rao-Blackwellized estimators to an equal number of evaluations of $g$, Rao-Blackwellization yields equal or lower variance, for a computable choice of $k$.
\end{itemize}

\section{Related Work}\label{sec:related_work}
Methods to reduce the variance of stochastic gradients for discrete distributions generally fall into two broad categories: continuous relaxation methods and control variate methods. 

In the first category, the Concrete distribution \cite{maddison2017concrete} approximates the discrete random variable with a reparametrizable continuous random variable so that the standard reparametrization trick can be applied. While this continuous relaxation reduces the variance of the stochastic gradient, the resulting estimators are biased. Thus the Gumbel-softmax procedure \cite{jang2017categorical} introduced an annealing step into the optimization whereby the continuous relaxation converges towards the discrete random variable as the optimization path moves forward. 

In the second category, control variate methods include black-box variational inference (BBVI) \cite{bbVI}, NVIL \cite{mnih2014neural}, DARN \cite{GregorDARN2014}, and MuProp \cite{GuMuProp2015}. BBVI uses multiple samples at each step to estimate the `optimal' control variate. NVIL introduces an observation dependent control variate learned by a separate neural network. DARN uses a Taylor expansion of $f_\eta(z)$ to compute a control variate, but this results in a biased estimator; MuProp proposes an estimate of this bias and corrects it. 

Finally, RELAX~\cite{grathwohl2018backpropagation} and REBAR~\cite{tucker2017rebar} are a combination of the two broad methods and use a continuous relaxation to construct a control variate. 

Section 5 compares both continuous relaxation and control variate methods to our Rao-Blackwellization.  

A Rao-Blackwellization procedure for gradient estimation was also applied in BBVI and ``local expectation gradients"~\cite{TitsiasLocalExpGrads}, but for a different purpose. In their setting, the expectation is decomposed over a multivariate (discrete or continuous) random variable using iterated expectation. BBVI approximates each conditional expectation by sampling (with a control variate), while local expectation gradients compute each conditional expectation analytically. This Rao-Blackwellization is orthogonal to our approach: while they consider multiple discrete random variables, our approach focuses on a univariate discrete with many categories. 

The process of summing out a few terms and sampling the remainder for gradient estimation has appeared in the context of reinforcement learning~\cite{TitsiasMCwithExhaustiveSearch, liangMAP2018}, though to our knowledge we are the first to make the connection with Rao-Blackwellization. In MAPO~\cite{liangMAP2018}, a procedure to create a memory
buffer of trajectories for policy optimization, the terms with high rewards (or small loss) are kept and summed. In contrast, we choose to sum terms with high probability. In our setting, it is the loss $f_\eta(z)$, not the probability, $q_\eta(z)$, that is expensive to evaluate for all categories $z$. 

Finally, the problem of having a large number of categories also manifests in language models, and methods such as noise contrastive estimation \cite{GutmannNCE2010} and hierarchical softmax \cite{MorinHierSoftmax2005} have been introduced. However, these methods are applied when the normalizing constant for $q_\eta(z)$ is intractable. In our work, we restrict ourselves to scenarios where $q_\eta(z)$ is normalized.


\section{Experiments}\label{sec:results}

In our experiments, we will consider applying the Rao-Blackwellization procedure to either the REINFORCE estimator, 
\begin{align}
\begin{split}
    g(z) = f_\eta(z) \nabla_\eta \log q_\eta(z) + \nabla_\eta f_\eta(z), \\\quad z\sim q_\eta(z),
    \label{eq:reinforce_no_bl}
\end{split}
\end{align}
or REINFORCE with a control variate $C$, 
\begin{align}
\begin{split}
    g(z) = [f_\eta(z) - C]\nabla_\eta \log q_\eta(z) + \nabla_\eta f_\eta(z), \\
    z\sim q_\eta(z). \label{eq:reinforce_simple_bl}
\end{split}
\end{align}
A simple choice of control variate that works well in practice is to take $C = f_\eta(z')$ for an independent draw $z' \sim q_\eta$. We abbreviate this estimator as REINFORCE$^+$.

Note that in both REINFORCE and REINFORCE$^+$, $g(z)$ is unbiased for the true gradient. (In the second case, $g(z)$ is unbiased conditional on $z'$, and hence unconditionally unbiased as well.)

\subsection{Bernoulli latent variables}\label{sec:toy_example}
We fix a vector $p = [0.6, 0.51, 0.48]^\top$ and seek to minimize the loss function
\begin{align}
\E_{b_1, b_2, b_3 \overset{iid}{\sim} \text{Bern}(\sigma(\eta))} \Big\{\sum_{i=1}^3 (b_i - p_i)^2\Big\}
\label{eq:bern_exp_loss}
\end{align}
over $\eta\in\mathbb{R}$, where $\sigma(\eta)$ is the sigmoid function. Here, the discrete random vector $b = [b_1, b_2, b_3]^\top$ is supported over $K=2^3 = 8$ categories. The optimal value of $\sigma(\eta)$ is $1$, approached as $\eta \to \infty$. 

Figure~\ref{fig:bernoulli_optim_paths} shows the performance of Rao-Blackwellizing REINFORCE and REINFORCE$^+$. We initialized $\eta$ at $\eta = -4$, so the sampling distribution has large mass at $b = (0, 0, 0)$. The optimal distribution on the other hand should put all mass at $b = (1, 1, 1)$. In other words, we initialized the optimization procedure such that the mass is concentrated on the wrong point. The Rao-Blackwellized gradient is therefore initially slightly slower than the original gradient, since we are analytically summing the wrong category. However, Rao-Blackwellization improves the performance of both gradient estimators at the end of the path. 

\begin{figure}[tb]
    \centering
     \begin{subfigure}[b]{0.4\textwidth}
        \includegraphics[width=\textwidth]{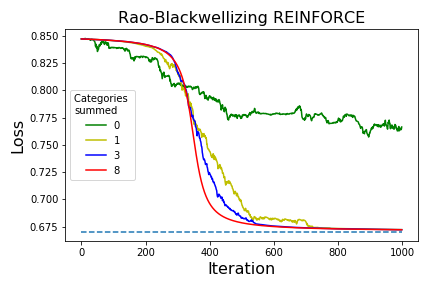}
    \end{subfigure}
    \begin{subfigure}[b]{0.4\textwidth}
        \includegraphics[width=\textwidth]{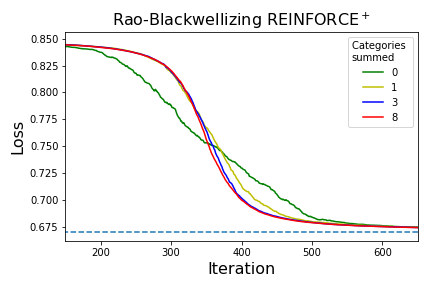}
    \end{subfigure}
    \vspace{-0.19in}
    \caption{The loss function at each iteration in the Bernoulli experiments. Each line is an average over 20 trials from the same initialization. Zero categories summed is the original estimator, while eight categories summed returns the exact gradient. }
    \label{fig:bernoulli_optim_paths}
\end{figure}

Figure \ref{fig:bernoulli_grad_vars} shows the variances of the gradient estimates at $\eta = 0$ and $\eta = -4$, as a function of $k$, the categories analytically summed. As expected, the variance decreases as more categories are analytically summed. At $\eta = 0$, the corresponding $q_\eta$ distribution is uniform, i.e., maximally anti-concentrated, so the variance reduction of Rao-Blackwellization is not large. However, the gains are quite substantial at $\eta = -4$, where $q_\eta$ is concentrated around the point $b = (0, 0, 0)$. In this case, analytically summing out one category removes nearly all the variance. 

\begin{figure}[tb]
    \centering
    \begin{subfigure}[b]{0.4\textwidth}
      \includegraphics[width=\textwidth]{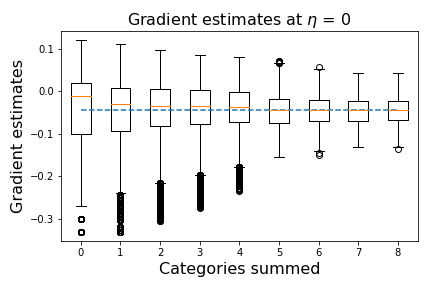}
    \end{subfigure}\\
    \begin{subfigure}[b]{0.4\textwidth}
        \includegraphics[width=\textwidth]{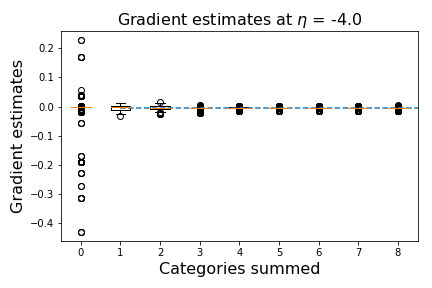}
    \end{subfigure}
    \caption{The distribution of gradient estimates from REINFORCE$^+$ in the Bernoulli experiments. We examine the gradients at $\eta = 0$ and $\eta = -4$, as a function of $k$, the number of categories summed. Summing out categories reduces variance. The reduction is large at $\eta=-4$ where the variational distribution is concentrated on just one category. (Note there is still some random noise when we sum out all 8 categories here, because of the random control variate.)}
    \label{fig:bernoulli_grad_vars}
\end{figure}

\subsection{Gaussian mixture model }\label{sec:gmm_results}
For our next experiment, we draw $N=200$ observations $(y_n)$ from a $d$-dimensional Gaussian mixture model with $K=10$ components, taking $d=2$. 
\begin{align}
	z_n &\stackrel{\text{iid}}{\sim} \text{Categorical}(\pi_{1:K}), \ n = 1, \dots, N, \\
	\mu_k & \stackrel{\text{iid}}{\sim} \mathcal{N}(0, \sigma_0^2 I_{d \times d}), \ k = 1, \dots, K, \\
	y_n | z_n, \mu & \stackrel{\text{iid}}{\sim} \mathcal{N}(\mu_{z_n}, \sigma_y^2I_{d\times d}), \ n = 1, \dots, N.
\end{align}
Here each $\mu_k$ is a Gaussian centroid and each $z_n$ is a cluster membership indicator.

As exact inference of the posterior $p(\mu, z | y)$ is intractable, we approximate it variationally~\cite{VI_review} with the mean-field family 
\begin{align}
q(\mu, &z) = \prod_{k=1}^{K}q(\mu_{k})
    \prod_{n=1}^{N}q(z_{n}).
\end{align}
Here 
\begin{align}
    q(\mu_{k}) &= \delta\{\mu_k = \hat\mu_k\}, \\
    q(z_{n}) &= \text{Categorical}\left(\hat\pi_n\right),
\end{align}
where $\delta\{\cdot = \hat\mu_k\}$ is the Dirac-delta function.

%

%


We then seek to minimize $\text{KL}(q(\mu, z) \| p(\mu, z | y))$ over the variational parameters $\hat\mu$ and $\hat\pi$. This is equivalent to maximizing the ELBO
\begin{align}
	\sum_{n=1}^N \E_{q(z_n; \pi_n)}\Big[&\log \frac{p(y_n | \hat\mu, z_n) p(z_n)}{q(z_n)}\Big] + \sum_{k=1}^K \log p(\hat\mu_k). \label{eq:gmm_objective}
\end{align}

Note that the expectation over $z_n$ is a summation over $K=10$ categories. Figure \ref{fig:gmm_results} compares the performance of unbiased stochastic gradients produced from REINFORCE$^+$ to the Rao-Blackwellization of REINFORCE$^+$ for optimization of the ELBO in Equation \eqref{eq:gmm_objective}.

Unlike the Bernoulli example, we are also optimizing parameters inside the expectation; specifically, in this case we are jointly optimizing the variational mean parameters $\hat\mu_k$ alongside the $\hat{\pi}_n$. We expect that more quickly learning the latent categories $z_n$ aids the optimization process, since the mean parameters depend on the cluster memberships. 

We initialized the optimization with $K$-means. Figure~\ref{fig:gmm_results} shows that Rao-Blackwellization improves the convergence rate, with faster convergence when more categories are summed. With summing just three categories, we nearly recover the same ELBO trajectory of the exact gradient, which sums all ten categories. We chose $K = 10$ as an example so we can compare against the exact gradient; with larger $K$, computing the exact gradient will become intractable and stochastic methods such as ours will be required.  

We also examine here the computational trade-off. Our Rao-Blackwellized estimator with $k$ categories summed requires $k + 1$ evaluations of the original REINFORCE$^+$ estimator. For a fairer comparison, we also consider the benefits of variance reduction obtained from simple Monte-Carlo sampling, where $k+1$ samples of the REINFORCE$^+$ estimator are averaged at each iteration. In this experiment, Rao-Blackwellization yields better performance than Monte-Carlo averaging. This is because for most observations, memberships are fairly unambiguous and so $q(z)$ is concentrated. This is the regime where our theory suggests significant variance reduction using Rao-Blackwellization. 

\begin{figure}[tb]
    \centering
    \hspace{0.3in}\begin{subfigure}[b]{0.3\textwidth}
        \includegraphics[width=\textwidth]{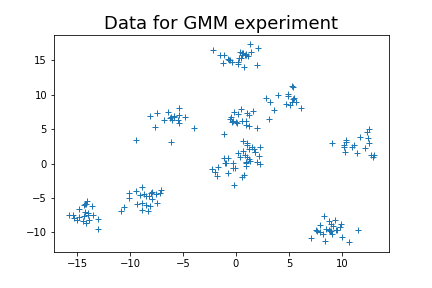}
    \end{subfigure}
    \begin{subfigure}[b]{0.4\textwidth}
        \includegraphics[width=\textwidth]{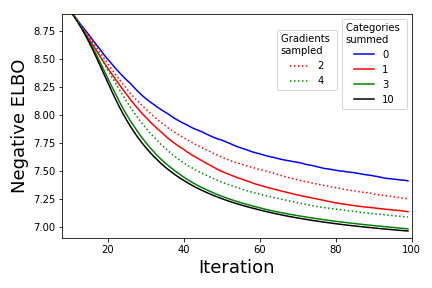}
    \end{subfigure}
    \vspace{-0.2in}
    \caption{Results for Gaussian mixture model experiment. (Top) Simulated data.
    (Bottom) Solid lines display the negative ELBO per iteration using REINFORCE$^+$, for $k$ categories summed. Zero categories summed is the original REINFORCE$^+$ estimator, while 10 categories summed returns the analytic gradient. 
    Dashed lines show performance when $n\in\{2, 4\}$ draws of the REINFORCE$^+$ estimator are averaged at each iteration to reduce variance. Each line is an average over 20 trials from the same initialization.
    }
    \label{fig:gmm_results}
\end{figure}


\subsection{Generative semi-supervised classification}
\label{sec:ss_mnist}
\subsubsection{Semi-supervised models}
The goal of a semi-supervised classification task is to predict labels $y$ from $x$, but where the training set consists of both labeled data $(x, y) \sim \mathcal{D}_L$ and unlabeled data $x \sim \mathcal{D}_{U}$. The approach proposed by Kingma et al.~(2014) uses a variational autoencoder (VAE) whose latent space is joint over a Gaussian variable $z$ and the discrete label $y$. The training objective is to learn a classifier $q_\phi(y | x)$, an inference model $q_\phi(z | x, y)$, and a generative model $p_\theta(x | y, z)$. On labeled data, the variational lower bound is 
\begin{align}
    \log p_\theta&(x, y) \geq \mathcal{L}^{L}(x, y)   \\ 
    &:= \E_{q_\phi(z|x, y)}[\log p_\theta(x | y, z) + \notag \\
    & \qquad \log p_\theta(z) + \log p_\theta(y) -
    \log q_\phi(z | x, y)]
    \label{eq:vi_bound_labeld}
\end{align}
On unlabeled data, the unknown label $y$ is treated as a latent variable and integrated out, 
\begin{align}
    \log p_\theta(x) &\geq \mathcal{L}^{U}(x) \\ 
    & := \E_{q_\phi(z|x, y)q_\phi(y | x)}[\log p_\theta(x | y, z) + \notag \\
    & \qquad \log p_\theta(z) + \log p_\theta(y) - \notag \\
    & \qquad \log q_\phi(z | x, y) - \log q_\phi(y | x)] \\
    &= \E_{q_\phi(y|x)}[\mathcal{L}^L(x, y) - \log q_\phi(y | x)]
    \label{eq:vi_bound_unlabeled}
\end{align}
The full objective to be maximized is 
\begin{align}
    \mathcal{J} = \E_{x \sim \mathcal{D}_U}[\mathcal{L}^{U}(x)] & +   
    \E_{(x, y) \sim \mathcal{D}_L}[\mathcal{L}^{L}(x, y)] \notag \\
    & + \alpha \E_{(x, y) \sim \mathcal{D}_L}[\log q_\phi(y | x)] 
\end{align}
where the third term is added for the classifier $q_\phi(y | x)$ to also train on labeled data. $\alpha$ is a hyperparameter which we set to 1.0 in our experiments. 

We take $z$ to be a continuous random variable with a standard Gaussian prior. Hence, gradients can flow through $z$ using the reparametrization trick. However, $y$ is a discrete label. The original approach proposed by Kingma et al.~(2014) computed the expectation in Equation~\eqref{eq:vi_bound_unlabeled} by exactly summing over the ten categories. However, most images are unambiguous in their classification, so $q_\phi(y | x)$ is often concentrated on just one category. We will show that applying our Rao-Blackwellization procedure with one category summed gives results comparable to computing the the full sum, more quickly. 

\subsubsection{Experimental Results}
We work with the MNIST dataset~\cite{LecunMNIST}. 
We used $50\,000$ MNIST digits in the training set, $10\,000$ digits in the validation set, and $10\,000$ digits in the test set. Among the $50\,000$ MNIST digits in the training set, $5\,000$ were randomly selected to be labeled, and the remaining $45\,000$ were unlabeled. 

To optimize, we Rao-Blackwellized the REINFORCE estimator. We compared against REINFORCE without Rao-Blackwellization; the exact gradient with all 10 categories summed; REINFORCE$^+$; Gumbel-softmax~\cite{jang2017categorical}; NVIL~\cite{mnih2016variational}; and RELAX~\cite{grathwohl2018backpropagation}. 

For all methods, we used performance on the validation set to choose step-sizes and other parameters. See Appendix for details concerning parameters and model architecture. 

\begin{figure}[tb]
    \centering
        \includegraphics[width=0.49\textwidth]{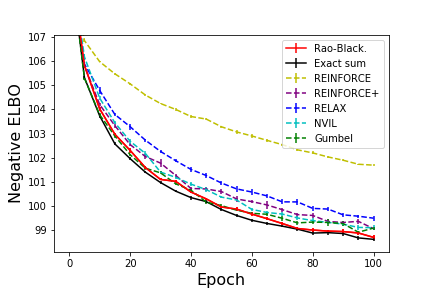}
    \vspace{-0.3in}
    \caption{Results on the semisupervised MNIST task. Plotted is test set negative ELBO evaluated at the MAP label. Paths are averages over 10 runs from the same initialization. Vertical lines are standard errors. Our method (red) is comparable with summing out all ten categories (black). }
    \label{fig:ss_mnist_result}
    \vspace{-0.15in}
\end{figure}

\begin{table}[tb]
\caption{Accuracies and timing results on semi-supervised MNIST classification. Standard errors of test accuracies are over 10 runs of each method. Standard deviations of timing are over the 100 epochs of 10 runs. Training was run on a p3.2xlarge instance on Amazon Web Services. }
\vspace{0.1in}
\begin{tabular}{l|c|c}
Method & test acc. (SE) & secs/epoch (SD)\\\hline 
RB-REINFORCE & 0.965 (0.001) & 17.5 (1.8) \\
Exact sum & 0.966 (0.001) & 31.4 (3.2) \\
REINFORCE & 0.940 (0.002) & 15.7 (1.6) \\
REINFORCE$^+$ & 0.953 (0.001) & 17.2 (1.7) \\
RELAX & 0.966 (0.001) & 29.8 (3.0) \\
NVIL & 0.956 (0.002) &  17.5 (1.8)\\
Gumbel-softmax &  0.954 (0.001) & 16.4 (1.7)
\end{tabular}
\label{tab:ss_mnist_accuracies}
\end{table}



Figure \ref{fig:ss_mnist_result} shows the negative ELBO, $-\mathcal{L}^L(x, y)$ from Equation~\eqref{eq:vi_bound_labeld}, on the test set evaluated at the MAP label as a function of epoch. In this experiment, our Rao-Blackwellization with one category summed (RB-REINFORCE) achieves the same convergence rate as the original approach where all ten categories are analytically summed. Moreover, our method achieves comparable test accuracy, at 97\%. Finally, our method requires about 18 seconds per epoch, compared to 31 seconds per epoch when using the full sum (Table \ref{tab:ss_mnist_accuracies}).

In comparing with other approaches, we clearly improve upon the convergence rate of REINFORCE. We slightly improve on RELAX. On this example, REINFORCE$^+$, NVIL, and Gumbel-softmax also give results comparable to ours. 



\subsection{Moving MNIST}
\label{sec:moving_mnist}
In this section, we use a hard-attention mechanism~\cite{mnihattnmechanism, GregorDRAW} to model non-centered MNIST digits. We choose this problem because, as will be seen, the exact stochastic gradient is intractable due to the large number of categories. However, only a few of the categories will have significant probabilities.

Like the original VAE work \cite{kingma2014auto}, we learn an inference model $q_\phi(z | x)$ and generative model $p_\theta(x | z)$, where $z$ is a low-dimensional, continuous representation of the MNIST digit $x$. Unlike the previous section, we are no longer using the class label. However, we now work with a non-centered MNIST digit, and in order to train the inference and generative models, we must also infer the pixel at which the MNIST digit is centered. 

More precisely, our generative model is as follows.
For each image, we sample a two-vector representing the pixel at which to center the original $28 \times 28$ MNIST image:
\begin{align}
    \ell \sim \mathrm{Categorical}(H \times W).
\end{align}
Here $H$ and $W$ are respectively the height and width, in pixels, of the larger image frame on which the MNIST digit will be placed. We take $H = W = 68$ in our experiments. 

Next, we generate the non-centered MNIST digit as 
\begin{align}
    z &\sim \mathcal{N}(0, I_d), \\
    x_{h,w} | \ell, z & \stackrel{ind}{\sim}
     \text{Bernoulli}(\mu(z)[h - \ell_0, w - \ell_1]),
     \label{eq:moving_mnist_bern_model}. 
\end{align}
 for $h\in\{0, ..., H-1\}$ and $w\in\{0, ..., W-1\}$. Here $\mu$ is a neural network that maps $z\in\mathbb{R}^d$ to a grid of mean parameters $\mu(z)\in \mathbb{R}^{28 \times 28}$. In Equation \eqref{eq:moving_mnist_bern_model}, we take $\mu(z)[a, b] = 0$ if $(a, b)\notin[0, 28]^2$. 

In this way, $x \in \mathbb{R}^{H \times W}$ is a random sample of an image containing a single non-centered MNIST digit on a blank background (Figure \ref{fig:moving_mnist_example}). 

Hence, we need to learn not only the generative model for an MNIST digit, but also the pixel at which the digit is centered. Our two latent variables are $z_n$ and $\ell_n$. We find a variational approximation to the posterior using an approximating family of the form 
\begin{align}
    \ell_n | x_n &\sim \mathrm{Categorical}(\zeta(x_n)), \\
    z_n | x_n, \ell_n &\sim \mathcal{N}(h_\mu(x_n, \ell_n), h_\Sigma(x_n, \ell_n)),
\end{align}
where $\zeta$, $h_\mu$, and $h_\Sigma$ are neural networks. The appendix details the architecture for the neural networks. 

REINFORCE was too high variance to be practical here, so we started with REINFORCE$^+$ and its Rao-Blackwellization. Here, we chose to sum the top five categories. We again compare with NVIL, Gumbel-softmax, and RELAX. For all the methods, we use a validation set to tune step-sizes and other parameters. 



Figure \ref{fig:moving_mnist_elbo} shows the negative ELBO on the test set evaluated at the MAP pixel location as a function of epoch. RELAX converged to a similar ELBO as our method, but did so at a slower rate. While NVIL also converged quickly, it converged to a worse negative ELBO than our method. 

Gumbel-softmax did not appear to converge to a reasonable ELBO. We believe that the bias of this procedure was too high in this application. In particular, because we are constrained to sampling discrete values for the pixel attention, we must use the straight-through version of Gumbel-softmax~\cite{StraightThrough, jang2017categorical}, which suffers from even higher bias. 

Our method is more computationally expensive per epoch than the others (Table~\ref{tab:moving_mnist_timing}). However, the gains in convergence are still substantive: for example, it takes about 44 seconds for our method to reach a negative ELBO of 500, while it takes RELAX about 110 seconds. 

Our method performs best because it is the only one that takes advantage of the fact that only a few digit positions have high probabilities. Summing these positions analytically removes much of the variance. 

\begin{figure}[tb]
    \centering
        \includegraphics[width=0.3\textwidth]{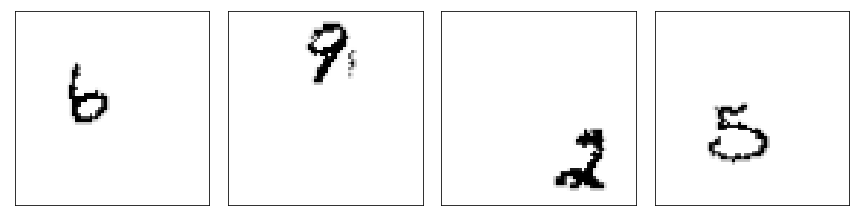}
     \vspace{-0.1in}
    \caption{Examples of non-centered MNIST digits}
    \label{fig:moving_mnist_example}
    \vspace{-0.15in}
\end{figure}

\begin{figure}[tb]
    \vspace{-0.12in}
    \centering
        \includegraphics[width=0.49\textwidth]{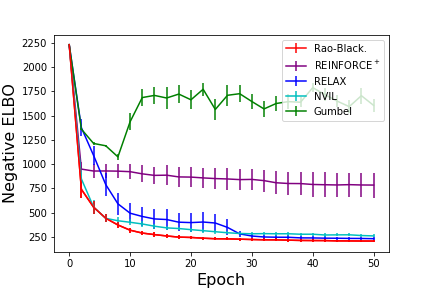}
    \vspace{-0.3in}
    \caption{Results on the moving MNIST task. Plotted is test set negative ELBO evaluated at the MAP pixel location. Paths are averages over 10 runs from the same initialization. Vertical lines are standard errors. Our Rao-Blackwellization (red) with summing out the top five categories exhibits the fastest convergence and reaches a smaller negative ELBO than NVIL and REINFORCE$^+$.}
    \label{fig:moving_mnist_elbo}
    \vspace{-0.2in}
\end{figure}

\begin{table}[tb]
\caption{Timing results on the moving MNIST task. Standard deviations of timing are over the 50 epochs of 10 runs. Training was run on a p3.2xlarge instance on Amazon Web Services. }
\begin{tabular}{l|r}
Method & secs/epoch (SD) \\\hline 
RB-REINFORCE$^+$ &  15.4 (2.3) \\
REINFORCE$^+$ &  \;8.9 (1.3) \\
RELAX & 11.1 (1.6) \\
NVIL & \;9.5 (1.4) \\
Gumbel-softmax & \;8.7 (1.2)
\end{tabular}
\label{tab:moving_mnist_timing}
\end{table}


\section{Discussion}\label{sec:discussion}
Efficient stochastic approximation of the gradient $\nabla_{\eta} \E_{q_{\eta}(z)} [f_\eta(z)]$, where $z$ is discrete, is a basic problem that arises in many probabilistic modelling tasks.
We have presented a general method to reduce the variance of stochastic estimates of this gradient, without changing the bias. Our method is grounded in the classical technique of Rao-Blackwellization. Experiments on synthetic data and two large-scale MNIST modeling problems show the practical benefits of our variance-reduced estimators.

We have focused on the particular setting where $z$ is a univariate discrete random variable, which is relevant for many applications. In other situations, multiple discrete variables will naturally appear in the expectations being optimized. Treating these as a single discrete variable over the Cartesian product of the sample spaces may make such problems amenable to our Rao-Blackwellization approach.

In addition, many multivariate discrete distributions arising in modeling applications will be structured (e.g., the discrete-space latent Markov chain of an HMM). Local expectation gradients \cite{TitsiasLocalExpGrads} reduce high-dimensional expectations over these multivariate discrete distributions to iterated univariate expectations through appropriate conditioning on variable sets. Our technique can then be applied for variance reduction in computing the univariate expectations. This is an avenue of future research.

\nocite{kingma2014semisupervised}
\bibliography{references}
\bibliographystyle{icml2019}

\appendix
\newpage

\section{An example with countably infinite $K$}
We give an example to demonstrate our method when there is a countably infinite number of categories. Consider the N-mixture model, 
\begin{align}
N &\sim \text{Poisson}(\lambda)\\
y_i &\sim \text{Binomial}(N, p) \quad \text{for} \quad i = 1, ..., n, 
\label{eq:binom}
\end{align}
a model used in ecological modeling of species counts~\cite{RoyleNmixtureModel}. 

In our experiment, we take $p$ and $\lambda$ to be known parameters. We want to infer $N$ given data $y_1, ..., y_n$. Since the support of $N$ is the integers greater than or equal to $y_{max}~:=~\max_n~\{y_n\}$, we use a negative binomial distribution shifted by $y_{max}$ to approximate the posterior. Let $\hat r$ and $\hat p$ be the number of failures and the probability of success, respectively, for a negative binomial. We optimize the ELBO, 
\begin{align}
\mathcal{L}(\hat r, \hat p) = E_{q(N; \hat r, \hat p)}[\log p(y | N)p(N) - \log q(N; \hat r, \hat p)]
\end{align}
This expectation is taken over $N$, and is given by an infinite sum. The exact expectation is intractable. However, we have a closed form variational distribution, and for any $\hat r$ and $\hat p$, it is easy to find the integers $N$ where $q(N; \hat r, \hat p)$ places most of its mass. We therefore can apply our Rao-Blackwellization procedure to compute stochastic gradients of the ELBO. 

In our experiment, we take the true $N = 10$ and $p = 0.2$. We drew 1000 data points from Equation~\eqref{eq:binom}. We set our Poisson prior with $\lambda = 10$. 

We found that the REINFORCE estimator was too high variance to be useful in this example, so we start with REINFORCE$^+$. Figure~\ref{fig:n_mixture_elbo_path} compares the REINFORCE$^+$ estimator with its Rao-Blackwellization, using either $k = 1$ or $k = 3$ categories summed. 

\begin{figure}[h!]
    \centering
        \includegraphics[width=0.4\textwidth]{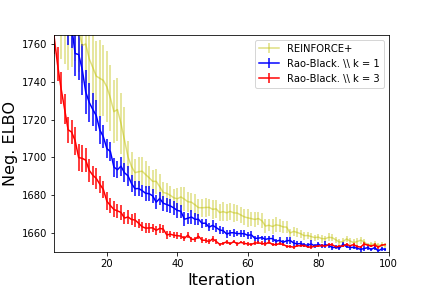}
    \caption{Negative ELBO per iteration in the N-mixture experiment. We compare the REINFORCE$^+$ estimator with its Rao-Blackwellization, using either $k = 1$ or $k = 3$ categories summed. Vertical lines denote standard errors over 10 trials from the same initialization. }
    \label{fig:n_mixture_elbo_path}
\end{figure}

We find that our Rao-Blackwellization improves the convergence rate of the ELBO. This is because our variational distribution eventually concentrates around the true $N$ (Figure~\ref{fig:n_mixture_q}), and only a few categories have significant mass. 
\begin{figure}[h!]
    \centering
        \includegraphics[width=0.4\textwidth]{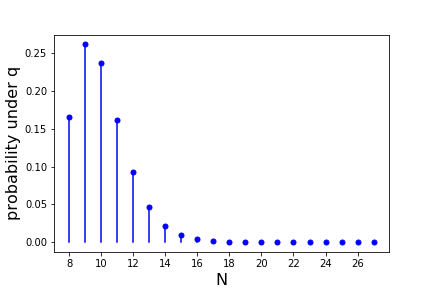}
    \caption{Negative binomial variational distribution $q$ at convergence for the N-mixture experiment. }
    \label{fig:n_mixture_q}
\end{figure}

\section{Experimental details}
Implementations of all methods in our paper as well as code to reproduce our results can be found in the git repository
\url{https://github.com/Runjing-Liu120/RaoBlackwellizedSGD}. 

\subsection{Generative semi-supervised classification}
In this experiment, our classifier $q_\phi(y|x)$ consists of three fully connected hidden layers, each with 256 nodes and ReLU activations. The inference and generative models, $q_\phi(z|x, y)$ and $p_\theta(x|z, y)$, both have one hidden layer with 128 nodes and ReLU activations, similar to the MLPs used in Kingma et al.~(2014). The latent variable $z$ is five dimensional and $q_\phi(z | x)$ is multivariate Gaussian with diagonal covariance.

For all methods, we used performance on a validation set to choose between the possible step-sizes, \{5e-5, 1e-4, 5e-4, 1e-3, 5e-3\}. For Gumbel-softmax, we also chose the annealing rate among \{1e-5, 5e-5, 1e-4, 5e-4\}. For RELAX, the relaxation temperature was chosen adaptively using gradients, while the scaling parameter was set at 1.0. 

The step-size for REINFORCE was chosen to be 1e-4 and the step-size for RELAX was chosen to be 5e-4. The step-size for the remaining methods were chosen to be 1e-3. The annealing rate for Gumbel-softmax was chosen to be 5e-4. 

Optimization was done with Adam~\cite{KigmaADAM2014}, with parameters $\beta_1=0.9$, $\beta_2 = 0.999$. An initialization for $q_\phi(z|x, y)$ and $p_\theta(x|z, y)$ was obtained by first optimizing $\mathcal{L}^L(x, y)$ on the labeled data only. We also initialized $q_\phi(y|x)$ on the labeled data using cross-entropy loss. The results in the paper show the optimization of the semi-supervised ELBO starting from this initialization. 

\subsubsection{Conditional generation results}
Figure~\ref{fig:ss_mnist_generation} displays the conditional generation of MNIST digits obtained after 100 epochs of running our Rao-Blackwellized gradient method. 

\begin{figure}[h!]
    \centering
    \begin{subfigure}[b]{0.2\textwidth}
        \includegraphics[width=\textwidth]{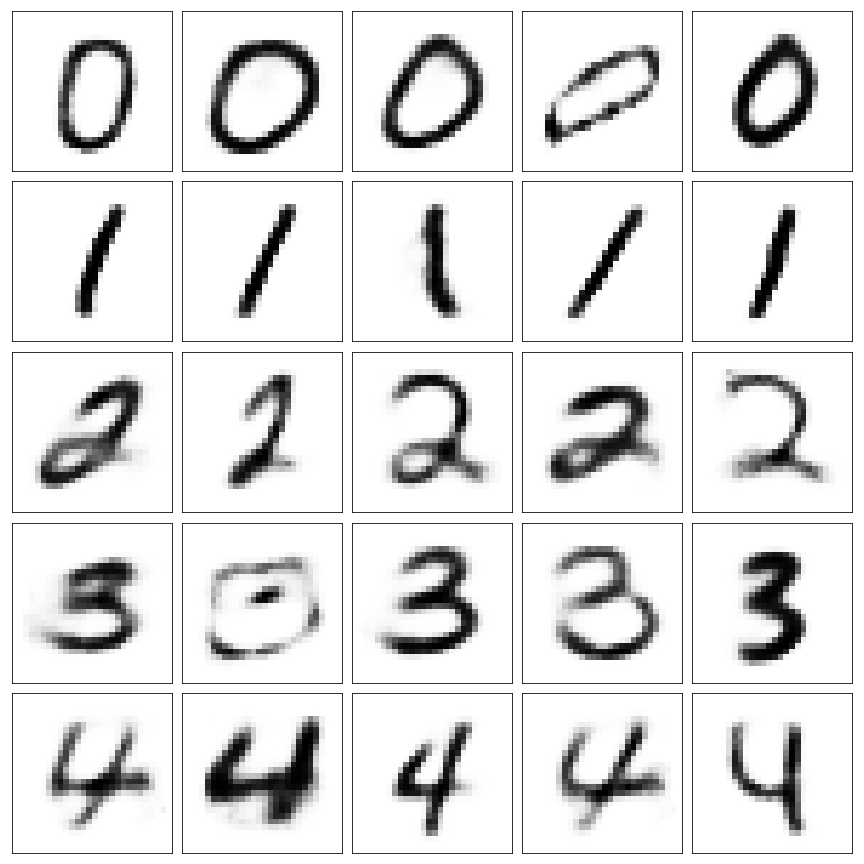}
    \end{subfigure}
    \qquad 
    \begin{subfigure}[b]{0.2\textwidth}
        \includegraphics[width=\textwidth]{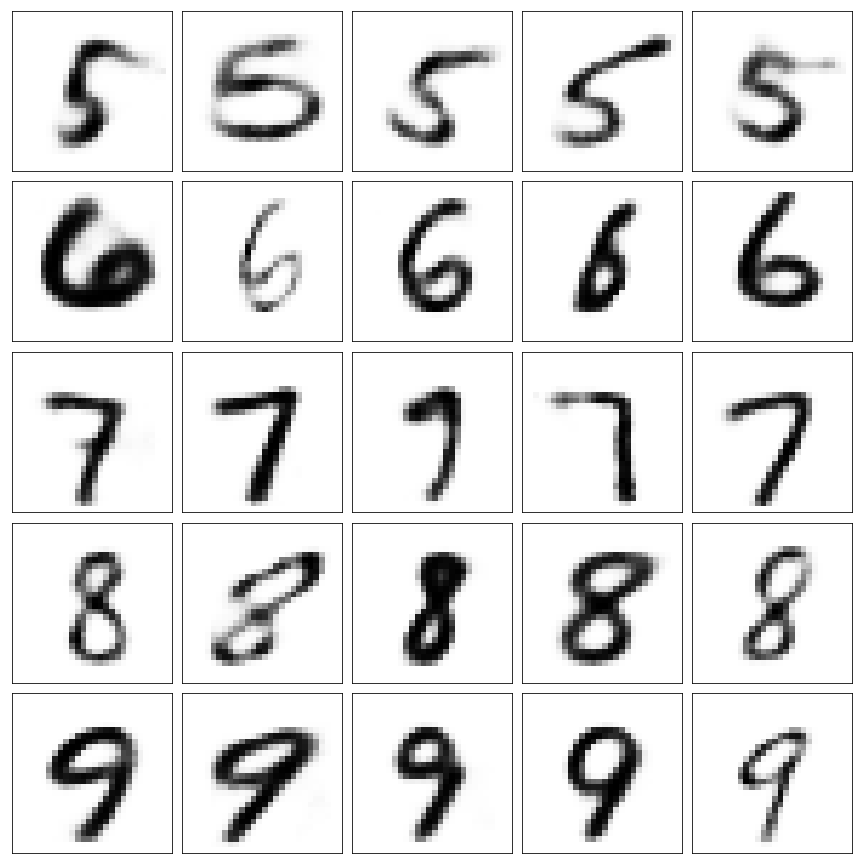}
    \end{subfigure}
    \caption{The conditional generation of MNIST digits. Each row displays five draws from the learned generative model
    $z\sim \mathcal{N}(0, I)$, $x \sim p_\theta(x | y, z)$, for a different digit $y$ in each row. }
    \label{fig:ss_mnist_generation}
\end{figure}

\subsection{Moving MNIST}

For the decoder $p(x | l, z)$ we use one fully connected hidden layer with 256 nodes and tanh activations, similar to the architecture described in Kingma and Welling (2014). Our $z$ is 5 dimensional. 

The attention mechanism $q(l | x)$ contains four convolutional layers, each with 7 output channels and ReLU activations; the final layer is a fully connected layer with a softmax. The encoder network $q(z|x)$ has one fully connected hidden layer with 256 nodes and tanh activations, mirroring the decoder network. 

We again used performance on the validation set to choose between the possible step-sizes and model parameters as described in the section above. The learning rate and annealing rate for Gumbel-sofmax was chosen to be 5e-5 and 5e-4, respectively.
For RELAX, the learning rate was 5e-4. The step-sizes for the remaining procedures were chosen to be 1e-3. We again use Adam~\cite{KigmaADAM2014} for optimization, and we set $\beta_1=0.9$, $\beta_2 = 0.999$. 

\subsubsection{VAE reconstruction}
Figure~\ref{fig:moving_mnist_recon} displays (1) the original non-centered MNIST digit; (2) the reconstruction of the MNIST digits after passing through our attention mechanism and VAE; and (3) the learned pixel locations. 

\begin{figure}[h!]
    \centering
        \includegraphics[width=0.4\textwidth]{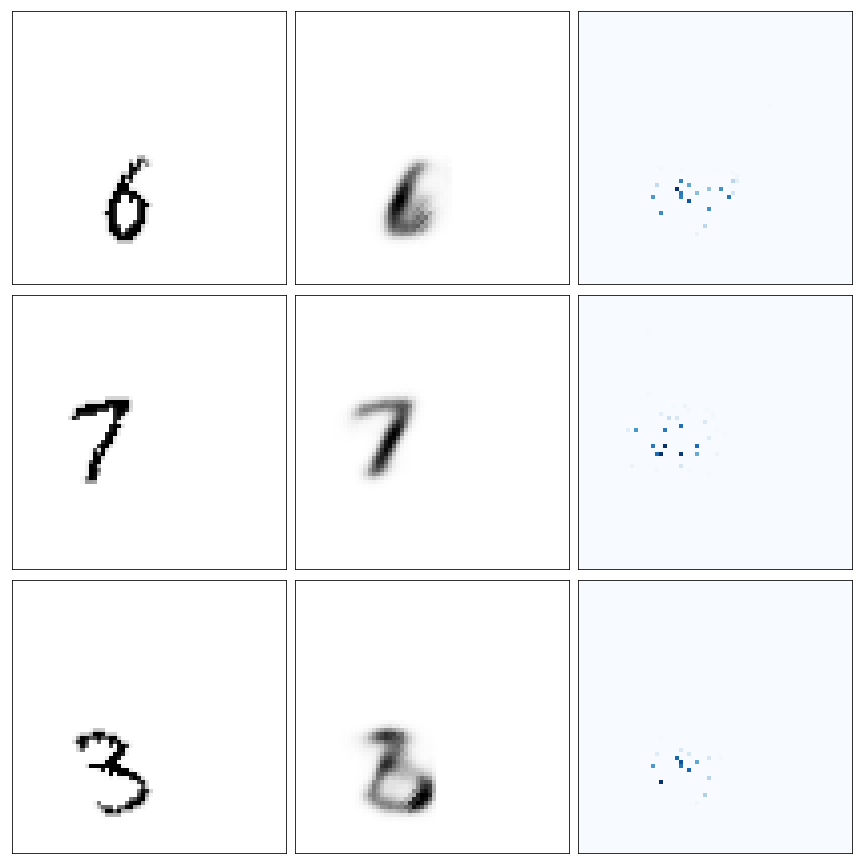}
    \caption{(Left column) The original MNIST digit. (Center column) The reconstructed MNIST digit. (Right column) The learned probability distribution over the grid of pixels. Brighter spots indicate higher probabilities.}
    \label{fig:moving_mnist_recon}
\end{figure}

\end{document}